\pdfoutput=1
\documentclass{article}

\usepackage[utf8]{inputenc} %
\usepackage[T1]{fontenc}    %

\usepackage[bitstream-charter]{mathdesign}
\usepackage{amsmath}
\usepackage[scaled=0.92]{PTSans}

\usepackage[
  paper  = letterpaper,
  left   = 1.65in,
  right  = 1.65in,
  top    = 1.0in,
  bottom = 1.0in,
  ]{geometry}

\usepackage[usenames,dvipsnames,table]{xcolor}
\definecolor{shadecolor}{gray}{0.9}

\usepackage[final,expansion=alltext]{microtype}
\usepackage[english]{babel}
\usepackage[parfill]{parskip}
\usepackage{afterpage}
\usepackage{framed}

{\endMakeFramed}

\DeclareRobustCommand{\parhead}[1]{\textbf{#1}~}

\usepackage{lineno}

\usepackage{ragged2e}

\newcommand{\parnum}{\bfseries\P\arabic{parcount}}
\newcounter{parcount}
\newcommand\p{%
    \stepcounter{parcount}%
    \leavevmode\marginpar[\hfill\parnum]{\parnum}%
}
\usepackage{graphicx}
\usepackage{wrapfig}
\usepackage[labelfont=bf,font=footnotesize,width=.9\textwidth]{caption}
\usepackage[format=hang]{subcaption}

\usepackage{booktabs,multirow,multicol}       %

\usepackage[algoruled,algo2e]{algorithm2e}
\usepackage{listings}
\usepackage{fancyvrb}
\fvset{fontsize=\normalsize}

\usepackage[colorlinks,linktoc=all]{hyperref}
\usepackage[all]{hypcap}
\hypersetup{citecolor=MidnightBlue}
\hypersetup{linkcolor=MidnightBlue}
\hypersetup{urlcolor=MidnightBlue}

\usepackage[nameinlink]{cleveref}

\usepackage[acronym,smallcaps,nowarn]{glossaries}
\lstdefinestyle{mystyle}{
    commentstyle=\color{OliveGreen},
    keywordstyle=\color{BurntOrange},
    numberstyle=\tiny\color{black!60},
    stringstyle=\color{MidnightBlue},
    basicstyle=\ttfamily,
    breakatwhitespace=false,
    breaklines=true,
    captionpos=b,
    keepspaces=true,
    numbers=left,
    numbersep=5pt,
    showspaces=false,
    showstringspaces=false,
    showtabs=false,
    tabsize=2
}
\usepackage[%
minnames=1,maxnames=99,maxcitenames=2,
style=alphabetic,
doi=false,
url=false,
firstinits=true,
hyperref,
natbib,
backend=bibtex,
sorting=nyt
]{biblatex}%

\DeclareFieldFormat{sentencecase}{\MakeSentenceCase*{#1}}

\renewbibmacro*{title}{%
  \ifthenelse{\iffieldundef{title}\AND\iffieldundef{subtitle}}
    {}
    {\ifthenelse{\ifentrytype{article}\OR\ifentrytype{inbook}%
      \OR\ifentrytype{incollection}\OR\ifentrytype{inproceedings}%
      \OR\ifentrytype{inreference}}
      {\printtext[title]{%
        \printfield[sentencecase]{title}%
        \setunit{\subtitlepunct}%
        \printfield[sentencecase]{subtitle}}}%
      {\printtext[title]{%
        \printfield[titlecase]{title}%
        \setunit{\subtitlepunct}%
        \printfield[titlecase]{subtitle}}}%
     \newunit}%
  \printfield{titleaddon}}

\AtEveryBibitem{%
\ifentrytype{article}{
    \clearfield{url}%
    \clearfield{urldate}%
    \clearfield{eprint}
    \clearfield{eid}
}{}
\ifentrytype{book}{
    \clearfield{url}%
    \clearfield{urldate}%
}{}
\ifentrytype{collection}{
    \clearfield{url}%
    \clearfield{urldate}%
}{}
\ifentrytype{incollection}{
    \clearfield{url}%
    \clearfield{urldate}%
}{}
}

\AtEveryBibitem{
    \clearfield{pages}
    \clearfield{review}%
    \clearfield{series}%
    \clearfield{volume}
    \clearfield{month}
    \clearfield{eprint}
    \clearfield{isbn}
    \clearfield{issn}
    \clearlist{location}
    \clearfield{series}
    \clearlist{publisher}
    \clearname{editor}
}{}

\usepackage{csquotes}
\usepackage[utf8]{inputenc} %
\usepackage[T1]{fontenc}    %
\usepackage{hyperref}       %
\usepackage{url}            %
\usepackage{booktabs}       %
\usepackage{nicefrac}       %
\usepackage{microtype}      %
\usepackage{xcolor}         %

\usepackage{amsmath}
\definecolor{shadecolor}{gray}{0.9}

\usepackage[english]{babel}

\DeclareRobustCommand{\parhead}[1]{\textbf{#1}~}

\usepackage{graphicx}
\usepackage{wrapfig}
\usepackage[labelfont=bf,font=footnotesize,width=.9\textwidth]{caption}
\usepackage[format=hang]{subcaption}

\usepackage{booktabs,multirow,multicol}       %

\usepackage[all]{hypcap}
\hypersetup{citecolor=MidnightBlue}
\hypersetup{linkcolor=MidnightBlue}
\hypersetup{urlcolor=MidnightBlue}

\usepackage[nameinlink]{cleveref}

\usepackage[acronym,smallcaps,nowarn]{glossaries}

\usepackage{centernot}
\usepackage{amsthm}
\usepackage{nicefrac}       %
\usepackage{mathtools}
\usepackage{amsbsy}
\usepackage{amstext}
\usepackage{amsthm}
\usepackage{thmtools}
\usepackage{thm-restate}

\begingroup
    \makeatletter
    \@for\theoremstyle:=definition,remark,plain\do{%
        \expandafter\g@addto@macro\csname th@\theoremstyle\endcsname{%
            \addtolength\thm@preskip\parskip
            }%
        }
\endgroup

\newcommand\independent{\protect\mathpalette{\protect\independenT}{\perp}}
\def\independenT#1#2{\mathrel{\rlap{$#1#2$}\mkern2mu{#1#2}}}

\newcommand{\V}{\mathbb{V}}

\newcommand{\g}{\, | \,}
\newcommand{\s}{\, ; \,}

\newcommand{\E}[2]{\mathbb{E}_{#1}\left[#2\right]}

\usepackage{booktabs,arydshln}
\makeatletter
\def\adl@drawiv#1#2#3{%
        \hskip.5\tabcolsep
        \xleaders#3{#2.5\@tempdimb #1{1}#2.5\@tempdimb}%
                #2\z@ plus1fil minus1fil\relax
        \hskip.5\tabcolsep}
\newcommand{\cdashlinelr}[1]{%
  \noalign{\vskip\aboverulesep
           \global\let\@dashdrawstore\adl@draw
           \global\let\adl@draw\adl@drawiv}
  \cdashline{#1}
  \noalign{\global\let\adl@draw\@dashdrawstore
           \vskip\belowrulesep}}
\makeatother

\renewcommand{\epsilon}{\varepsilon}

\declaretheorem[style=plain,name=Theorem]{theorem}

\declaretheorem[style=definition,sibling=theorem,name=Definition]{defn}

\declaretheorem[style=definition,sibling=theorem,name=Example]{example}
\declaretheorem[style=remark,sibling=theorem,name=Remark]{remark}

\newenvironment{example*}
 {\pushQED{\qed}\example}
 {\popQED\endexample}
\numberwithin{equation}{section}

\newcommand{\defnphrase}[1]{\emph{#1}}

\DeclareMathOperator*{\argmin}{argmin}

\renewcommand{\Pr}{\mathrm{P}}

\newcommand{\given}{\mid}

\newcommand{\dist}{\ \sim\ }

\providecommand\given{} %
\newcommand\SetSymbol[1][]{
  \nonscript\,#1:\nonscript\,\mathopen{}\allowbreak}
\DeclarePairedDelimiterX\Set[1]{\lbrace}{\rbrace}%
{ \renewcommand\given{\SetSymbol[]} #1 }

\crefname{defn}{Definition}{Definitions}

\usepackage[separate-uncertainty=true,multi-part-units=single]{siunitx} %
\usepackage[]{tikz}
\usepackage{paralist}
\usepackage{enumitem} %
\usetikzlibrary{positioning}

\usepackage[edges]{forest}

\declaretheoremstyle[
spacebelow=\parsep,
    spaceabove=\parsep,
  mdframed={
    backgroundcolor=gray!10!white,     %
    hidealllines=true, 
    innertopmargin=8pt, 
    innerbottommargin=4pt, 
    skipabove=8pt,
    skipbelow=10pt,
    nobreak=true
}
]{grayboxed}

\usepackage{thm-restate}

\usepackage{xcolor}
\usepackage[affil-it]{authblk}

\newcommand{\N}{\mathbb{N}}
\newcommand{\Z}{\mathbb{Z}}
\newcommand{\Etrain}{\mathcal{E}_{\text{train}}}

\renewcommand{\E}{\mathcal{E}}

\newcommand{\pa}{\mathrm{pa}}

\begin{document}

\title{The Causal Structure of Domain Invariant Supervised Representation Learning}
\date{}
\author[1]{Zihao Wang}
\author[1,2]{Victor Veitch}
\affil[1]{Department of Statistics, University of Chicago}
\affil[2]{Google Research}

\maketitle

\begin{abstract}
  Machine learning methods can be unreliable when deployed in domains that differ from the domains on which they were trained. 
  There are a wide range of proposals for mitigating this problem by learning representations that are ``invariant'' in some sense.
  However, these methods generally contradict each other, and none of them consistently improve performance on real-world domain shift benchmarks.
  There are two main questions that must be addressed to understand when, if ever, we should use each method.
  First, how does each ad hoc notion of ``invariance'' relate to the structure of real-world problems? And, second, when does learning invariant representations actually yield robust models?
  To address these issues, we introduce a broad formal notion of what it means for a real-world domain shift to admit invariant structure.
  Then, we characterize the causal structures that are compatible with this notion of invariance.
  With this in hand, we find conditions under which method-specific invariance notions correspond to real-world invariant structure, and we clarify the relationship between invariant structure and robustness to domain shifts. 
  For both questions, we find that the true underlying causal structure of the data plays a critical role.
\end{abstract}

\section{Introduction}

Machine learning methods are unreliable in the presence of \emph{domain shift}, where there is a mismatch between the environment(s) where training data is collected and the environment where the trained model is actually deployed \citep{shimodaira2000improving,quinonero2008dataset}.
A variety of techniques have been proposed to mitigate domain shift problems.
One popular approach is to try to learn a representation function $\phi$ of the data such that $\phi(X)$ is in some sense ``invariant'' across domains.
The motivating intuition is that $\phi(X)$ should preserve the structure of the data that is common across domains, while throwing away the part that varies across domains.
It then seems intuitive that a predictor trained on top of such a representation would have robust performance in new domains.
In this paper, we'll study two closely related questions: 
When do different ``invariant'' learning methods actually succeed at learning the part of the data that is invariant across domains? And, When does learning a domain invariant representation actually help with robustness of out-of-domain predictions?

There are many methods aimed at domain-invariant representation learning.
We'll focus on supervised methods, that learn representations useful for predicting some label $Y$.
For example, consider the following broad families of supervised domain-invariant representation learning methods:
\begin{description}\label{sec:methods-category}
\item[Data Augmentation] We perturb each input in some way and learn a representation that is the same for all perturbed versions.
  E.g., if $t(X)$ is a small rotation of an image $X$, then we require $\phi(X)=\phi(t(X))$ \citep{krizhevsky2012imagenet, hendrycks2019augmix, cubuk2019autoaugment, xie2020unsupervised, wei2019eda, paschali2019data, hariharan2017low, sennrich2015improving, kobayashi2018contextual, nie2020named}.

\item[Marginal Distribution Invariance] We require $\phi(X)$ to have the same distribution in all environments, $\Pr^e(\phi(X))=\Pr^{e'}(\phi(X))$, where $\Pr^e$ is the distribution in environment $e$ \citep{muandet2013domain, ganin2016domain, albuquerque2020adversarial, li2018domain, sun2017correlation, sun2016deep, matsuura2020domain}

\item[Conditional Distribution Invariance] We require the label-conditional distribution to be the same in each environment, $\Pr^e(\phi(X) \given Y)=\Pr^{e'}(\phi(X) \given Y)$ \citep[][]{li2018deep, long2018conditional, tachet2020domain, goel2020model}

\item[Sufficiency Invariance] We require that the representation is sufficient for predicting $Y$ in each environment, in the sense that $\Pr^e(Y \given \phi(X)) \! =\! \Pr^{e'}(Y \given \phi(X))$  \citep[][]{peters2016causal, rojas2018invariant, wald2021calibration}. 

\item[Risk Minimizer Invariance] We learn a representation $\phi(X)$ so that there is a fixed (domain-independent) predictor $w^*$ on top of $\phi(X)$ that minimizes risk in all domains \citep[][]{arjovsky2019invariant, lu2021nonlinear, ahuja2020invariant, krueger2021out, bae2021meta}.  
\end{description}
In each case, the intuitive aim is to learn a representation that throws away information that varies `spuriously' across domains while preserving information that is invariant across domains.
However, the notion of invariance is substantively different in each approach---indeed, they often directly contradict one another! 
Further, it is not obvious a priori which notion of invariance is the ``right'' one to use in any given problem.
The empirical situation is no better.
Although each of these methods improves robustness in some situations, none of them dominate on real-world domain shift benchmarks \cite{gulrajani2020search,koh2021wilds}. 
In fact, none of them even consistently beat naive empirical risk minimization!
For any particular problem, it is unclear which, if any, of these methods is appropriate.

The principle challenge is that it is unclear what ``part of the data that is invariant across domains'' should mean formally. So, it's not clear which of the method-specific notions of invariance to rely on.

In this paper, we'll come at the problem from the opposite direction. We begin by looking for an abstract notion of domain shift that aims to capture the broad intuition for why invariant structure should exist. That is, a notion of domain shift where the intuitive meaning of ``invariant'' has a precise formalization. Then, we'll characterize how each ad-hoc ``invariant'' learning method behaves under these kinds of domain shift.

More precisely, the development of this paper is as follows.
\begin{enumerate}
  \item We introduce \emph{Causally Invariant with Spurious Associations} (CISA) as an abstract notion of domain shift that admits a cannonical formalization of domain-invariant part of the data. Then, we characterize the real-world causal structures compatible with CISA. 
  \item With this in hand, we establish conditions under which each domain-invariant learning method succeeds (or fails) at learning the cannonical invariant structure. These conditions closely relate to the underlying causal structure of the problem.
  \item Additionally, we study the relationship between learning the invariant structure and achieving out-of-domain robustness. It turns out the relationship between these two things depends on the true underlying causal structure of the data. Depending on the structure, invariance may lead to robustness, or offer no guarantees at all.
\end{enumerate}

\section{Causally Invariant Domain Shifts}
The first problem is to specify what invariance means.

To anchor our discussion, let's consider an (idealized) real-world example.
\Citet{Beede:Retinopathy:2020} describe deploying a deep learning model to screen for diabetic retinopathy, a diabetes-related condition where blood vessels in the retina become damaged. The model's input is a photo of the patient's retina (taken on-site by a nurse), and its output is a predicted probability that the patient has the condition. The model was deployed across 11 clinics in Thailand. The authors find significant performance differences across clinics.\footnote{Particularly with respect the model rejecting cases it is uncertain about.}
They identify several possible causes of this variation. These include variable levels of lighting in the room where the photo is taken, degradation of the cameras, and inconsistent use of dilating eye drops.
Each of these factors affects the retinal photo that is fed to the network. 

In this example, it's intuitively clear what invariance should mean. We want to be able to process the photo in a manner that's insensitive to factors of variation such as lighting, camera degradation, and eye drop use. Our goal is an abstract, general, notion of invariance that captures this intuition.

\paragraph{Invariant part of $X$}

The first step is to introduce a formalization of the factors of variation that we want to be invariant to.
We model these as latent variables $Z$ that cause the observed features $X$---e.g., the lighting level in the room is a cause of the (brightness of the) photo.
We model these causes as \emph{latent} variables because we do not know what they are a priori. Our ultimate goal is to learn to be invariant to these factors using only environment information, without needing to explicitly identify them. 

Now, for invariance to be a sensible goal, $Z$ must not also be a cause of $Y$. 
However, for there to even be a problem with vanilla empirical risk minimization, there must be some statistical association between $Y$ and $Z$. If such an association did not exist, then we wouldn't need a method to explicitly enforce invariance. For example, it may be the case that it's more likely for the room light to be left on when the patient's eyesight is worse---this would create an association between $Y$ and $Z$.
\begin{defn}
  We say a latent cause $Z$ of $X$ is a \emph{spurious factor of variation} if it is not a cause of $Y$ and $Y$ and $Z$ are not independent. Call the set of all such causes the \emph{spurious factors of variation}.
\end{defn}
Since they're unobserved, we'll collapse the spurious factors of variation into a single variable $Z$ going forward. 

The next step is to specify what is meant by part of the features invariant to these spurious factors of variation.
These are aspects of the image that are not affected by lighting level, camera degradation, etc. 
Generally, these will be high-level, abstract features $g(X)$---e.g., whether there's a damaged blood vessel in the image.
We'll formalize this using definitions from \cite{veitch2021counterfactual}.
\begin{defn}
Let $g$ be a function of $X$. We say $g$ is \defnphrase{counterfactually invariant to spurious factors} (abbreviated CF-invariant), if $g(X(z)) = g(X(z')) \ a.s.,  \ \forall z, z' \in \mathcal{Z}$. Here, $X(z)$ is potential outcomes notation, denoting the $X$ we would have seen had $Z$ been set to $z$.
\end{defn}
A high-level feature $g(X)$ is not affected by $Z$ if it is CF-invariant to $Z$.
We then formalize the part of $X$ not affected by $Z$ as the collection of all such features:
\begin{defn}
The \defnphrase{invariant part of $X$}, $X_z^{\perp}$, is a $X$-measurable variable such that $g$ is CF-invariant iff $g(X)$ is $X_z^{\perp}$-measurable, for all measurable functions $g$.\footnote{Such a variable exists under weak conditions; e.g., $Z$ discrete \citep{veitch2021counterfactual}.}
\end{defn}

We can now view $X$ as divided into two parts: $X_z^\perp$ the part of $X$ not affected by the spurious factors $Z$, and $X_z$ the remaining part affected by $Z$. We do not assume the division of $X$ is known a priori (indeed $Z$ is not even known). Also note that $X_z^\perp$ may affect $X_z$, so these parts need not be independent. 

\paragraph{Domain Shift}
Now, we want a notion of domain shift that is compatible with this notion of invariance.
Morally, we want all causal relationships to be preserved across domains, while allowing the non-causal association between $Y$ and $Z$ to fluctuate.
E.g., the propensity for patients with retinopathy to have damaged blood vessels is the same across sites, but the propensity for the room light to be left for hard-of-seeing patients may vary.
To capture this structure, we introduce an additional latent variable $U$ that is a common cause of $Z$ and $Y$. We will allow the distribution of $U$ to change across domains, thereby inducing a change in the non-causal association between $Z$ and $Y$.
\begin{defn}
  We say an unobserved variable $U$ is an \emph{unobserved confounder} if it is a common cause of $Z$ and $Y$, does not confound the relationship between $X_z^\perp$ and $Y$ and is not caused by any other variable. 
\end{defn}

With this in hand, we can specify our notion of domain shift. We abstract each domain $e$ as a probability distribution $P_e$ over the observable variables $X$ and $Y$. Then,
\begin{defn}\label{def:cisa}
  We say a set of domains $\{P_e\}_{e \in \E}$ are \emph{Causally Invariant with Spurious Associations} (CISA), if there are unobserved spurious factors of variation $Z$ and unobserved confounder $U$
  so that 
  \begin{equation*}
    P_e(X, Y) = \int P_0(X, Y, Z | U) P_e(U) dZ dU, \forall e \in \E,
  \end{equation*}
  where $P_0$ is some fixed distribution and $P_e(U)$ is a domain-specific distribution of the unobserved confounder.   
\end{defn}

Summarizing, CISA gives an abstract notion of domain shift where ``invarant part of $X$'' is clearly defined.
CISA captures the following chain of reasoning. For ``invariant part'' of $X$ to make sense, there needs to be some latent factor of variation $Z$ to be invariant to. For invariance to be a sensible goal, $Z$ shouldn't causally affect $Y$. But, $Z$ and $Y$ need to be statistically associated, otherwise even a naive predictor won't use. Then, it's natural to assume that $Z$ and $Y$ are associated due to some unknown common cause $U$. Finally, to get a domain shift, we let the strength of association between $Y$ and $Z$ vary (by varying the distribution of $U$) across environments.

\subsection{CISA compatible causal structures}
\begin{figure}
  \vspace{-5pt}
  \centering
    \captionsetup{width=\columnwidth}
    \begin{subfigure}[b]{.3\textwidth}
      \centering
      \captionsetup{width=\linewidth}
      \scalebox{0.65}{
        \begin{tikzpicture}[
    var/.style={draw,circle,inner sep=0pt,minimum size=0.8cm},
    latentconf/.style={draw,circle,dashed,inner sep=0pt,minimum size=0.8cm, fill=green!5}
    ]
        \node (X) [var] {$X_z$};
        \node (Y) [var, right=0.7cm of X] {$Y$};
        \node (Z) [latentconf, left=0.7cm of X] {$Z$};
        \node (Xzperp) [var, above=0.7cm of X] {$X_z^{\perp}$};
        \node (U) [latentconf, below=0.7cm of X] {$U$};
        \node (E) [var, below=0.7cm of Y] {$E$};
        \path[->, line width=0.5mm] (Z) edge (X); 
        \path[->, line width=0.5mm] (U) edge (Z); 
        \path[->, line width=0.5mm] (U) edge (Y); 
        \path[->, line width=0.5mm] (Y) edge (X); 
        \path[->, line width=0.5mm] (Y) edge (Xzperp); 
        \path[->, line width=0.5mm] (Xzperp) edge (X); 
        \path[->, line width=0.5mm] (E) edge (U); 

        \tikzstyle{bigbox}=[draw, dotted, very thick, inner sep=8pt]
        \node[bigbox, fit=(X)(Xzperp)] (Xfull) {};
        \node[below right] at (Xfull.north west) {$X$};

    \end{tikzpicture}}
      \caption{anti-causal}
      \label{fig:anticausal}
  \end{subfigure}
  \begin{subfigure}[b]{.3\textwidth}
      \centering
      \captionsetup{width=\linewidth}
      \scalebox{0.65}{
      \begin{tikzpicture}[
    var/.style={draw,circle,inner sep=0pt,minimum size=0.8cm},
    latentconf/.style={draw,circle,dashed,inner sep=0pt,minimum size=0.8cm, fill=green!5}
    ]
        \node (X) [var] {$X_z$};
        \node (Y) [var, right=0.7cm of X] {$Y$};
        \node (Z) [latentconf, left=0.7cm of X] {$Z$};
        \node (Xzperp) [var, above=0.7cm of X] {$X_z^{\perp}$};
        \node (U) [latentconf, below=0.7cm of X] {$U$};
        \node (E) [var, below=0.7cm of Y] {$E$};
        \path[->, line width=0.5mm] (Z) edge (X); 
        \path[->, line width=0.5mm] (U) edge (Z); 
        \path[->, line width=0.5mm] (U) edge (Y); 
        \path[->, line width=0.5mm] (Y) edge (X); 
        \path[->, line width=0.5mm] (Xzperp) edge (Y); 
        \path[->, line width=0.5mm] (Xzperp) edge (X); 
        \path[->, line width=0.5mm] (E) edge (U); 

        \tikzstyle{bigbox}=[draw, dotted, very thick, inner sep=8pt]
        \node[bigbox, fit=(X)(Xzperp)] (Xfull) {};
        \node[below right] at (Xfull.north west) {$X$};
\end{tikzpicture} }
      \caption{conf-out}
      \label{fig:confout}
  \end{subfigure}
  \begin{subfigure}[b]{.3\textwidth}
      \centering
      \captionsetup{width=\linewidth}
      \scalebox{0.65}{    
      \begin{tikzpicture}[
    var/.style={draw,circle,inner sep=0pt,minimum size=0.8cm},
    latentconf/.style={draw,circle,dashed,inner sep=0pt,minimum size=0.8cm, fill=green!5}
    ]
        \node (X) [var] {$X_z$};
        \node (Y) [var, right=0.7cm of X] {$Y$};
        \node (Z) [latentconf, left=0.7cm of X] {$Z$};
        \node (Xzperp) [var, above=0.7cm of X] {$X_z^{\perp}$};
        \node (U) [latentconf, below=0.7cm of X] {$U$};
        \node (E) [var, below=0.7cm of Y] {$E$};
        \path[->, line width=0.3mm] (Z) edge (X); 
        \path[->, line width=0.3mm] (Xzperp) edge (X);
        \path[->, line width=0.3mm] (E) edge (U); 
        \path[->, bend left=45, line width=0.3mm] (Y) edge (Z); 
        \path[->, line width=0.3mm] (Xzperp) edge (Y);
        \path[->, bend left=90, line width=0.3mm] (U) edge (Xzperp); 
        \path[->, line width=0.3mm] (U) edge (Z); 
    
        \tikzstyle{bigbox}=[draw, dotted , very thick, inner sep=8pt]
        \node[bigbox, fit=(X)(Xzperp)] (Xfull) {};
        \node[below right] at (Xfull.north west) {$X$};
    \end{tikzpicture} }
      \caption{conf-desc}
      \label{fig:confdesc}
  \end{subfigure}
  \caption{Examples of CISA compatible causal structures. The spurious factors of variation $Z$ are latent, and unknown. The confounding factors $U$ are also latent. The observed features decompose into two parts, one causally affected by $Z$, and the other not. In CISA, the environment-invariant part of $X$ is $X_z^\perp$.}
  \label{fig:example-dags}
       \vspace{-15pt}
\end{figure}
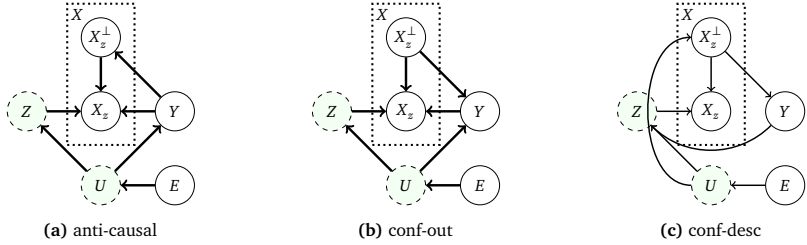
\Cref{def:cisa} is abstract, so it may not be intuitively clear what kinds of structures are compatible with CISA. 
A key requirement of CISA is that a common causal structure holds across domains.
\Cref{fig:example-dags} shows some examples of causal DAGS that are compatible with CISA. 
Note we have introduced a new variable $E$ to label the environment.
These all have the common structure that there is a part of $X$ that is invariant to the factors of variation $Z$ that have unstable cross-environment relationships with $Y$.
But they differ in other ways, and it's not clear that an invariant learning method that works for one should also work for others (indeed, this will be a major theme of the subsequent sections). 

We briefly describe these examples to give a sense of what kind of situations CISA can accomodate.

\paragraph{Anti-causal} The retinopathy example has anti-causal structure. Both the spurious factors $Z$ and true condition status $Y$ affect the image $X$. The strength of non-causal association between $Y$ and $Z$ varies across clinics. 

From \cref{fig:anticausal} we see that $\Pr(X_z^\perp \given Y)$ is invariant across domains. We can read this off because CISA preserves causal relationships between variables and their causal parents.
(It could also be derived directly from \cref{def:cisa} and the Markov property of the causal DAG)
If $Z$ was trivial, this would correspond to prior shift, where $\Pr(X \given Y)$ is assumed invariant \cite{lemberger2020primer}.

\paragraph{Confounded descendant}
Consider predicting disease status $Y$ from a large collection of medical tests $X$ (e.g., blood work, body mass, etc.). The tests serve as proxies for both causes and effects of the disease,
all of which may be confounded with patient demographics $Z$.
The nature of this confounding may change between hospitals, which we can view as multiple environments.

From \cref{fig:confdesc}, we see that $\Pr(Y \given X_z^\perp)$ is invariant across domains. In this case, $X_z^{\perp}$ is exactly the causal parents of $Y$.
So, the invariant predictor is one that uses only the causal parents of $Y$.
This ``learn the causal parents'' structure is the explicit goal of several causal invariant representation learning methods \citep{peters2016causal,arjovsky2019invariant}. In this sense, CISA allows for such structure.
If $Z$ was trivial, this would correspond to covariate shift, where $\Pr(Y \given X)$ is assumed invariant \cite{lemberger2020primer}.

\paragraph{Confounded outcome}
Consider trying to predict the quality of a product review. In training, we have review text $X$ and whether the review was judged to be ``helpful'' by at least one reader, $Y$. The text is affected by attributes $Z$ of the writer that may be causally unrelated to review quality---e.g., how upbeat they are.
Consider deploying across multiple product types; e.g., books and electronics. These are our environments. It may be the case that people who engage with book reviews are more upbeat then people who do the same for electronics. More upbeat people are more likely to judge a review to be helpful. Thus, there is confounding between the upbeatness of the review writer, $Z$, and the label $Y$. And, the strength of this association can vary across product types (environments).

From \cref{fig:confout}, we see that $\Pr(X_z^\perp)$ is invariant across domains. We can read this off because there is no directed path from $U$ to $X_z^\perp$ in the causal DAG. 
If $Z$ was trivial, this would correspond to the assumption that $\Pr(X)$ is invariant across environments.

\paragraph*{CISA compatible causal structures}
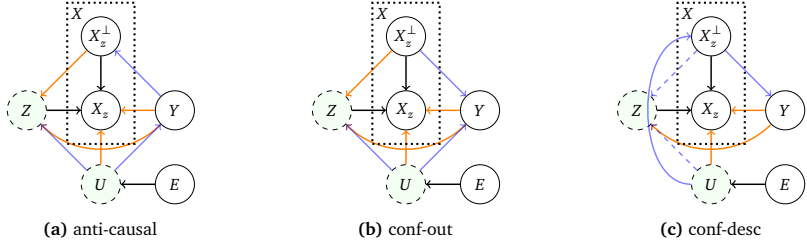
\begin{figure}
  \vspace{-5pt}
  \centering
    \captionsetup{width=\columnwidth}
    \begin{subfigure}[b]{.3\textwidth}
      \centering
      \captionsetup{width=\linewidth}
      \scalebox{0.65}{
        \begin{tikzpicture}[
var/.style={draw,circle,inner sep=0pt,minimum size=0.8cm},
latentconf/.style={draw,circle,dashed,inner sep=0pt,minimum size=0.8cm, fill=green!5}
]
    \node (X) [var] {$X_z$};
    \node (Y) [var, right=0.7cm of X] {$Y$};
    \node (Z) [latentconf, left=0.7cm of X] {$Z$};
    \node (Xzperp) [var, above=0.7cm of X] {$X_z^{\perp}$};
    \node (U) [latentconf, below=0.7cm of X] {$U$};
    \node (E) [var, below=0.7cm of Y] {$E$};
    \path[->, line width=0.3mm] (Z) edge (X); 
    \path[->, line width=0.3mm] (Xzperp) edge (X);
    \path[->, line width=0.3mm] (E) edge (U); 
    \path[->, orange, line width=0.3mm] (U) edge (X); 
    \path[->, orange, line width=0.3mm] (Y) edge (X); 
    \path[->, orange, bend left=45, line width=0.3mm] (Y) edge (Z); 
    \path[->, blue, line width=0.3mm, draw opacity=0.5] (Y) edge (Xzperp);
    \path[->, blue, line width=0.3mm, draw opacity=0.5] (U) edge (Y); 
    \path[->, blue, line width=0.3mm, draw opacity=0.5] (U) edge (Z); 
    \path[->, orange, line width=0.3mm] (Xzperp) edge (Z); 

    \tikzstyle{bigbox}=[draw, dotted, very thick, inner sep=8pt]
    \node[bigbox, fit=(X)(Xzperp)] (Xfull) {};
    \node[below right] at (Xfull.north west) {$X$};
\end{tikzpicture} }
      \caption{anti-causal}
      \label{fig:dags-types-a}
  \end{subfigure}
  \begin{subfigure}[b]{.3\textwidth}
      \centering
      \captionsetup{width=\linewidth}
      \scalebox{0.65}{
      \begin{tikzpicture}[
var/.style={draw,circle,inner sep=0pt,minimum size=0.8cm},
latentconf/.style={draw,circle,dashed,inner sep=0pt,minimum size=0.8cm, fill=green!5}
]
    \node (X) [var] {$X_z$};
    \node (Y) [var, right=0.7cm of X] {$Y$};
    \node (Z) [latentconf, left=0.7cm of X] {$Z$};
    \node (Xzperp) [var, above=0.7cm of X] {$X_z^{\perp}$};
    \node (U) [latentconf, below=0.7cm of X] {$U$};
    \node (E) [var, below=0.7cm of Y] {$E$};
    \path[->, line width=0.3mm] (Z) edge (X); 
    \path[->, line width=0.3mm] (Xzperp) edge (X);
    \path[->, line width=0.3mm] (E) edge (U); 
    \path[->, orange, line width=0.3mm] (U) edge (X); 
    \path[->, orange, line width=0.3mm] (Y) edge (X); 
    \path[->, orange, bend left=45, line width=0.3mm] (Y) edge (Z); 
    \path[->, blue, line width=0.3mm, draw opacity=0.5] (Xzperp) edge (Y);
    \path[->, blue, line width=0.3mm, draw opacity=0.5] (U) edge (Y); 
    \path[->, blue, line width=0.3mm, draw opacity=0.5] (U) edge (Z); 
    \path[->, orange, line width=0.3mm] (Xzperp) edge (Z); 

    \tikzstyle{bigbox}=[draw, dotted, very thick, inner sep=8pt]
    \node[bigbox, fit=(X)(Xzperp)] (Xfull) {};
    \node[below right] at (Xfull.north west) {$X$};
\end{tikzpicture} }
      \caption{conf-out}
      \label{fig:dags-types-b}
  \end{subfigure}
  \begin{subfigure}[b]{.3\textwidth}
      \centering
      \captionsetup{width=\linewidth}
      \scalebox{0.65}{    
      \begin{tikzpicture}[
var/.style={draw,circle,inner sep=0pt,minimum size=0.8cm},
latentconf/.style={draw,circle,dashed,inner sep=0pt,minimum size=0.8cm, fill=green!5}
]
    \node (X) [var] {$X_z$};
    \node (Y) [var, right=0.7cm of X] {$Y$};
    \node (Z) [latentconf, left=0.7cm of X] {$Z$};
    \node (Xzperp) [var, above=0.7cm of X] {$X_z^{\perp}$};
    \node (U) [latentconf, below=0.7cm of X] {$U$};
    \node (E) [var, below=0.7cm of Y] {$E$};
    \path[->, line width=0.3mm] (Z) edge (X); 
    \path[->, line width=0.3mm] (Xzperp) edge (X);
    \path[->, line width=0.3mm] (E) edge (U); 
    \path[->, orange, line width=0.3mm] (U) edge (X); 
    \path[->, orange, line width=0.3mm] (Y) edge (X); 
    \path[->, orange, bend left=45, line width=0.3mm] (Y) edge (Z); 
    \path[->, blue, line width=0.3mm, draw opacity=0.5] (Xzperp) edge (Y);
    \path[->, blue, bend left=90, line width=0.3mm, draw opacity=0.5] (U) edge (Xzperp); 
    \path[->, blue, dashed, line width=0.3mm, draw opacity=0.5] (U) edge (Z); 
    \path[->, blue, dashed, line width=0.3mm, draw opacity=0.5] (Xzperp) edge (Z); 

    \tikzstyle{bigbox}=[draw, dotted, very thick, inner sep=8pt]
    \node[bigbox, fit=(X)(Xzperp)] (Xfull) {};
    \node[below right] at (Xfull.north west) {$X$};
\end{tikzpicture} }
      \caption{conf-desc}
      \label{fig:dags-types-c}
  \end{subfigure}
  \caption{Every CISA compatible set of domains obeys one of the causal DAGs defined as follows: The DAG must match one of the three templates, where the black edges must be included, and the non-edges must be excluded. Orange edges may be included or excluded. In the case of \Cref{fig:dags-types-c}, at least one of the two dashed black arrows must be included. (Note that in \Cref{fig:dags-types-a} and \Cref{fig:dags-types-b}, the edges between $X_z^\perp$ and $Y$ are typically included, as otherwise there is no part of $X$ that has a non-trivial and stable relalationship with $Y$.)}
  \label{fig:dags-types}
       \vspace{-15pt}
\end{figure}

In fact, the three examples we just discussed are essentially the only causal structures compatible with CISA.
\begin{restatable}{theorem}{classifications} 
  A set of domains $\{P_e\}_{e \in \E}$ satisfies CISA if and only if the common underlying causal structure has a causal DAG that belongs to the set given in \Cref{fig:dags-types}. In particular, there are three families of allowed DAGs: anti-causal, confounded-outcome, or confounded-descendant.
\end{restatable}

\section{Domain Invariant Representation Learning}

We now return to the question of what the right notion of domain-invariant representation is. In the case of CISA domains,
there is a canonical notion for the part of $X$ that has a domain-invariant relationship with $Y$.
Namely, $X_z^\perp$, the part of $X$ that is not affected by the spurious factors of variation. 
Accordingly, the goal of domain-invariant representation learning in this context is to find a representation $\phi$ such that
$\phi(X)$ only depends on $X_z^\perp$. Formally, this means finding a representation $\phi$ such that $\phi(X)$ is counterfactually invariant. 

Of course, we could satisfy this condition by simply throwing away all of the information in $X$ (e.g., $\phi(X)$ is a constant everywhere).
So, we further look for the counterfactually-invariant representation that preserves the most predictive power for the label $Y$.
Let $\Phi_{\text{cf-inv}}(\E)$ denote the set of CF-invariant representations for CISA domains $\E$.
Then, we define the optimal CISA invariant representation as the solution to:
\begin{align}\label{eq:optimal-cisa-rep}
  \min_{\phi: \mathcal{X} \rightarrow \mathcal{H}, w: \mathcal{H} \rightarrow \mathcal{Y}} & E_{P_{\Etrain}}[L(Y, (w \circ \phi)(X))]\\
  \text{subject to \ } & \phi \in \Phi_{\text{cf-inv}}(\E)
\end{align}
Here, the predictor $w$ and loss function $L$ capture the sense in which $\phi(X)$ should be predictive of $Y$.

This equation defines the idealized goal of domain-invariant representation learning in CISA domains.
The challenge is that the spurious factors of variation are unknown and unobserved, so we cannot identify the set of 
counterfactually-invariant representations directly. 
However, having established the ideal goal, we can now assess existing methods in terms of how well they approximate this idea.
That is, we now ask: under what circumstances do existing invariant supervised representation methods approximate \cref{eq:optimal-cisa-rep}?

\subsection{Data Augmentation}
Data augmentation is a standard technique in machine learning pipelines, and has been shown to (sometimes) help when faced with domain shifts \citep{wiles2021fine}.
Our goal now is to understand when and why data augmentation might enable CISA domain-invariant representation learning.

The basic technique first applies pre-determined ``label-preserving'' transformations $t$ to original features $X$ to generate artificial data $t(X)$.
There are two ways this transformed data can be used. The first option is to simply add the transformed data as extra data to a standard learning procedure.
Alternatively, we might pass in pairs $(X_i, t(X_i))$ to our learning procedure, and directly enforce some condition that $\phi(X) \approx \phi(t(X))$ \citep{garg2019counterfactual,von2021self}.

We first formalize a notion of ``label-preserving'' for CISA domains.
The key idea is that we can think of transformation $t(X)$ of $X$ as being equivalent to changing some cause of $X$ and then propagating this change through.
For example, suppose a particular transformation $t$ rotates the input images by 30 degrees, and $Z$ is the factor of variation corresponding to the angle away from vertical.
Then, we can understand the action of $t$ as $t(X(z)) = X(z+30)$, where we again use the potential outcomes notation for counterfactuals.
With this idea in hand, we see that a transformation is \emph{label-preserving} in CISA domains if it is equivalent to a change that affects only spurious factors of variation.
That is, label-preserving transformations cannot affect $X_z^{\perp}$.
Otherwise, the transformation may change the invariant relationship with $Y$; changing the lighting level of a retinopathy image is ok, but changing number of broken blood vessels is not.    
\begin{defn}
  We say a data transformation $t: \mathcal{X} \rightarrow \mathcal{X}$ is \emph{label-preserving} for CISA domains $\E$ if, for each $X(z)$ there is $z'$ so that $t(X(z)) = X(z'), a.e.$.
\end{defn}

Label preserving transformations leave the CISA invariant relationships (between $X_z^\perp$ and $Y$) alone, but can change the relationship between $Y$ and the spurious factors of variation $Z$.
Intuitively, if we have a `large enough' set of such transformations, they can destroy the relationship between $Y$ and $Z$ that exists in the training data.
So, we might expect that training with such data augmentation will automatically yield an optimal counterfactually-invariant predictor.

This is nearly correct, with the caveat that things can go wrong if there is a part of $X$ causally related to both $Z$ and $Y$.
That is, if there is a part of $X$ that relies on the interaction between $Z$ and $Y$.
We follow \citet{veitch2021counterfactual} in formalizing how to rule out this case: 
\begin{defn}
The spurious factors of variation $Z$ are \emph{purely spurious} if $Y \independent X | X_z^{\perp}, Z$
\end{defn}
\vspace{-5pt}
We can now state the main result connecting data augmentation and domain-invariance:

\begin{restatable}{theorem}{dataaug}\label{thm:data-aug}
  For a CISA domain, if the set of transformations $\mathcal{T}$ satisfies label-preserving and enumerates all potential outcomes of $Z$, and either
  \begin{compactenum}
    \item the model is trained to minimize risk on augmented data, and $Z$ is purely spurious, or
    \item the model is trained to minimize risk on original data, with hard consistency regularization (i.e. enforcing $\phi(X) = \phi(t(X)), \forall t \in \mathcal{T}$),
  \end{compactenum}
  then we recover the CF-invariant predictor that minimizes risk on original data.  
\end{restatable}

Thus for CISA domains, ideal data augmentation (i.e. all possible label-preserving transformations) will exactly learn CF-invariant representations.
Moreover, this holds irrespective of what the true underlying causal structure is.
Accordingly, such data augmentation would be the gold standard for domain-invariant representation learning.

However, in practice, we cannot satisfy the idealized conditions.
Applying predefined transformations without thinking about the specific applications can lead to violations of the label preserving condition. For example in a bird classification task, changing color may really change the bird species (this is called manifold intrusion in \citet{guo2019mixup}). Further, heuristic transformations often cannot enumerate all potential outcomes of $X(z)$. Indeed, deep models can sometimes memorize predetermined transformations and fail to generalize with previously unseen transformations \citep{vasiljevic2016examining, geirhos2018generalisation}. 

Considering the limitations of using handcrafted transformations, a natural idea is to replace them with transformations learned from data.
However, in practice, $Z$ is unknown and we only observe the data domains $E$. Then, learning transformations must rely either on detailed structural knowledge of the problem \citep[e.g.,][]{robey2021model}, or on some distributional relationship between $E$, $Y$ and $X$ and $t(X)$ \citep[e.g.,][]{goel2020model}. Since $t(X)$ is only used for the representation learning, this is equivalent to learning based on some distributional criteria involving $E$, $Y$, and $\phi(X)$---the subject of the next section.

\subsection{Distributionally Invariant Learning}

Many domain-invariant representation learning methods work by enforcing some form of distributional invariance.
To explain these succinctly, we introduce a random variable $E$ to label the environment in which the data is observed;
so $X, Y \given E=e \dist P_e$. 

There are three notions of distributional invariance studied in the literature.
The first is $\phi(X) \independent E$ \citep[e.g.,][]{muandet2013domain,ganin2016domain}. The intuition is that this will preserve only features that cannot discriminate between training and test domains.
The second notion is $\phi(X) \independent E | Y$ \citep[e.g.,][]{li2018deep,long2018conditional}.
This allows us to have some features that are dependent on the environment, but only if they are redundant with the label.
Finally, some causally-motivated approaches have considered learning representations so that $Y \independent E | \phi(X)$ \citep[e.g.,][]{peters2016causal, wald2021calibration}.

The question is: when, if ever, are each of these distributional invariances the right approach for domain-invariant representation learning?
In the CISA setting, the following theorem provides an answer in terms of the shared causal structure of the environments.
\begin{restatable}{theorem}{signature}\label{thm:distributional-signature}
  Suppose $\phi$ is a CF-invariant representation. 
  \begin{compactenum}
    \item if the causal graph is anti-causal,  $\phi(X) \independent E | Y$;
    \item if the causal graph is conf-outcome,  $\phi(X) \independent E$;
    \item if the causal graph is conf-descendant,  $Y \independent E | \phi(X)$.
  \end{compactenum}
\end{restatable}
\begin{remark}
This theorem looks similar to \citet[][Thm.~3.2]{veitch2021counterfactual}. This is deceptive; here we observe the environment $E$, whereas they assumed observations of the spurious factors $Z$. 
\end{remark}

In words: each of the distributional invariances arises as a particular implication of CF-invariance.
In general, the distributional invariance is a weaker condition than CF-invariance.
In this sense, distributional invariance relaxes of CF-invariance; i.e., $\Phi_{\text{cf-inv}}(\E) \subset \Phi_{\text{DI}}(\Etrain)$.
However, we can directly use observed data to measure whether a distributional invariance is satisfied. 
Then, we can enforce distributional invariances at training time.
Indeed, this is exactly what the distributional invariance methods do.

This suggests a strategy for relaxing the CISA invariant learning problem, \cref{eq:optimal-cisa-rep}. 
Namely, for a given problem, determine the common causal structure, select the distributional invariance implied by that causal structure, and then learn a representation that satisfies that distributional invariance.
That is, we learn according to:  
\begin{align*}
  \min_{\phi \in \Phi_{\text{DI}}(\Etrain), w} & E_{P_{\Etrain}}[L(Y, (w \circ \phi)(X))],
\end{align*}
where $\Phi_{\text{DI}}(\E)$ is the set of representations that satisfy the independence criteria.

In summary: distributionally invariant learning methods are relaxations of the idealized CISA invariant learning, \cref{eq:optimal-cisa-rep}. This relaxation is justified only when we use the distributional invariance that matches the underlying causal structure.
Indeed, enforcing the wrong distributional invariance will directly contradict CF-invariance, and may \emph{increase} dependency on the spurious factors of variation. Distributional invariance is only a relaxation when we actually get the underlying causal structure correct!

\subsection{Invariant Risk Minimization}\label{ssec:irm}

The Invariant Risk Minimization (IRM) paradigm \citep{arjovsky2019invariant} aims to find representations that admit a single predictor that has the optimal risk across all domains. That is, the set of IRM representations is
\begin{equation*}
  \Phi_{\text{IRM}}(\E) := \{\phi: \exists \ w \text{ st } w \in \argmin_{\bar{w}} E_{P_e}[L(Y, (\bar{w} \circ \phi)(X))] \  \forall e \}
\end{equation*}
Here, the question is: when, if ever, does the IRM procedure correspond to a relaxation of the CISA procedure \cref{eq:optimal-cisa-rep}.
Again the answer will turn out to depend on the underlying causal structure.

\paragraph{Confounded Descendant}
IRM is justified in the case where $X$ includes both causes and descendants of $Y$, and the invariant predictor should use only information in the parents of $Y$.
As explained above, this coincides with the CISA invariance---and with distributional invariance---in the confounded-descendant case. 
Indeed, we can view IRM as a relaxation of $Y \independent E | \phi(X)$. Instead of asking the distribution $P^e(Y | \phi(X) = h)$ to be invariant across domains for every $h \in \mathcal{H}$, we only require the risk minimizer under $P^e(Y | \phi(X) = h)$  to be invariant. 
So, a partial answer is that IRM is a relaxation of distributional invariance (and CF-invariance) when the underlying causal structure is confounded descendant.

\paragraph{Confounded Outcome}
However, the situation is harsher in the case of other causal structures. For anti-causal and confounded-outcome problems, the typical case is $\Phi_{\text{IRM}}(\E) = \emptyset$. 
For confounded outcome, the idea of having the same risk minimizer across domains does not make sense without further assumptions. The general case has $Y \leftarrow f(X, U, \eta)$ where $\eta$ is noise independent of $X, U$. For example when $U = E$ we can write $Y \leftarrow f_E(X, \eta)$ so there could be arbitrarily different relationships between $X$ and $Y$ in every domain.
It may be possible to circumvent this by making structural assumptions on the form of $f$; e.g., there is an invariant risk minimizer in the case where the effect of $U$ and $X$ is additive \citep{veitch2021counterfactual}. 

\paragraph{Anti-Causal Case}
In the anti-causal case, there is also usually no invariant predictor. 
Here, $\Pr(X_z^\perp \given Y)$ is invariant across domains.
However, the risk of a predictor in domain $e$ will also depend on $\Pr_e(Y)$, which need not be invariant.
That is, there is no invariant risk minimizer because of \emph{prior shift} \citep{zhang2013domain}. 

However, it turns out there is nevertheless an important connection between IRM and CISA in the anti-causal case. 
To make this clear, we introduce a generalization of IRM that is appropriate for this setting.
The issue is that $P_e(Y)$ can shift across domains, so we simply define the risk minimizer under a fixed reference distribution $P_0(Y)$.
\begin{defn}
We define the set of representations satisfying \defnphrase{g-IRM} as: 
\begin{align*}
  &\Phi_{\text{g-IRM}}(\E, P_0) := \\
  &\{\phi: \exists \ w \text{ st } w  \in \argmin_{\bar{w}} E_{P_e}[\frac{P_0(Y)}{P_e(Y)}L(Y, (\bar{w} \circ \phi)(X))] , \forall e \}
\end{align*}
where $P_0(.)$ is some reference distribution for $Y$. 
\end{defn}

It turns out this generalization of IRM is a relaxation of CISA in the anti-causal case.
\begin{restatable}{theorem}{irm}\label{thm:irm}
Let $\{P_e\}_{e \in \E}$ satisfy CISA, then 
\begin{compactenum}
  \item if $\{P_e\}_{e \in \E}$ is confounded-descendant, then $\Phi_{\text{DI}}(\E) \subset \Phi_{\text{IRM}}(\E)$ 
  \item if $\{P_e\}_{e \in \E}$ is anti-causal, then $\Phi_{\text{DI}}(\E) \subset \Phi_{\text{g-IRM}}(\E, P_0)$ for any chosen $P_0$ 
\end{compactenum}
\end{restatable}

\paragraph{IRM with Rebalancing}
This result tells us that we can approximate the idealized CISA procedure by using the variant of IRM that matches the true causal structure of the problem.
Importantly, the difference between the two IRM variants is equivalent to whether we standardize the prior distribution $P_e(Y)$ across environments (anti-causal case) or not (confounded descendant).
Interestingly, this standardization is already done routinely in practice when using IRM (often undocumented). 
It turns out that this is not a benign optimization trick, but a fundamental change in the assumed underlying causal structure!
Indeed, in \cref{sec:experiment}, we show that this standardization is critical even for the demonstration toy example in the original IRM paper! And, if we use a version of IRM mismatched to the underlying causal structure, we fail to learn a good representation.

\subsection{Summary}
We now summarize the main results of this section.

   \textbf{Idealized data augmentation is the gold standard} if it's possible to enumerate all label-preserving transformations. In this case, data augmentation always yields the ideal invariant representation, irrespective of the causal structure of the problem \Cref{thm:data-aug}. This enumeration is impossible in general as we don't even know what the spurious factors $Z$ are. Still, using augmentation with label-preserving transformations (but not exhaustively) enforces a relaxation of the idealized invariant learning.
   For example, perturbing photo brightness in the retinopathy example might help with sensitivity to lighting level. 

   \textbf{Distributional invariance} relaxes CF-invariance if and only if chosen to match the underlying causal structures (\Cref{thm:distributional-signature}). This can be a good option when full augmentation is not possible. In the retinopathy example,
   we now see that the conditional distribution invariance $\phi(X) \independent E \given Y$ would help reduce sensitivity to the spurious factors (it's an anti-causal problem), and the other distributional invariances would not.

   \textbf{(Generalized) IRM further relaxes distributional invariance} for anti-causal and confounded-descendant problem, when it's chosen to match the causal structure of the problem (\Cref{thm:irm}). It weakens the full independence criteria to use just the implication for a single natural test statistic: the loss of the model. 

\section{Insights for Robust Prediction}\label{sec:insights}
Often, learning domain-invariant representations is an intermediate step towards learning robust predictors.
To do this, we train a predictor $\hat{w}$ on top of the invariant representation $\phi(X)$ using data from the training domains.
Then, given an example with features $x$, in a new domain, we predict $y$ using $\hat{w}(\phi(x))$.
The hope is that the use of the invariant representation will make the predictor robust across domains. 
We now discuss the implications of our domain-invariant representation learning results for robust prediction.

\subsection{The relationship between invariance and robustness depends on the underlying causal structure}
\Cref{ssec:irm} characterizes when a predictor trained on top of an invariant CISA representation will be a risk minimizer in all domains.
This is a reasonable ideal for what robust to domain shift might mean formally.

\paragraph{Confounded Descendant}
For confounded-descendant problems, we have seen that the two concepts, ``risk minimizer in all domains'' and ``optimal predictor on top of invariant representation'' are the same.
So, for this causal structure, invariant representation learning yields robust prediction.

\paragraph{Confounded Outcome}
The general situation is more complicated.
Without further structural assumptions, we can't say anything about the confounded outcome case. 
This is fundamental: even after we banish the effect of spurious factors of variation, there is still need not be any fixed predictor that is stable across environments.

\paragraph{Anti-Causal}
For anti-causal problems, there is no invariant risk minimizer in general because $P_e(Y)$ can change across domains.
However, this is essentially the only thing that can go wrong. If we train a predictor $w$ on top of an invariant representation $\phi$ in domain $e$, then its risk in a new CISA domain $e'$ is:
\begin{equation}
  E_{P_e}[\frac{P_{e'}(Y)}{P_e(Y)}L(Y, (w \circ \phi)(X))].
\end{equation}
That is, the risk can only inflate proportionate to how much the distribution of $Y$ changes across domains.
Accordingly, if the prior distribution doesn't shift very much between training and deployment, then the training-domain optimal predictor trained on the invariant representation will be nearly optimal in the test environment. This kind of limited label shift seems common in practice.
Acccordingly, invariant representation learning will often lead to robust prediction in the anti-causal setting.

\subsection{Further Insights}

We make some additional observations on how we might go about using domain-invariant representation learning for robust prediction.

\paragraph{Data augmentation helps in most cases}
Label-preserving data augmentation won't hurt domain generalization and can often help. This is true no matter the underlying causal structure of the problem. This matches empirical benchmarks where data augmentations usually help domain generalization performance, sometimes dramatically \citep{wiles2021fine,koh2021wilds,gulrajani2020search}. For example, \citet{wiles2021fine} finds that simple augmentations used in \cite{krizhevsky2012imagenet} generally improves performance when "augmentations approximate the true underlying generative model". 

\paragraph{Pick a method matching the true causal structure}
Many papers apply distributional invariance approaches with no regard to the underlying causal structure of the problem. In particular, many tasks in benchmarks have the anti-causal structures, but the methods evaluated do not include those enforcing $\phi(X) \independent E | Y$ \cite{koh2021wilds}. Indeed, \Citet{tachet2020domain} find that methods enforcing $\phi(X) \independent E | Y$ consistently improve over methods that enforce $\phi(X) \independent E$---retrospectively, this is because they benchmark on anti-causal problems. \Citet{wiles2021fine} finds that learned data augmentation \citep{goel2020model} consistently improves performance in deployment. This method can be viewed as enforcing $\phi(X) \independent E | Y$ and, again, the benchmarks mostly have anti-causal structure.

\section{Related Work}\label{sec:related-work}
\parhead{Causal invariance in domain generalization} Several works \citep{peters2016causal,rojas2018invariant,arjovsky2019invariant, lu2021nonlinear} specify domain shifts with causal models. 
In these approaches, the set of environments is modeled as those achievable by intervening on \emph{any} node in the causal graph, other than the label $Y$.
This creates a large set of possible environments.
As such, the set of invariant predictors is small. 
(Fewer functions are invariant across more environments.)
In this setting, the optimal invariant predictor turns out to be $P(Y \given \pa(Y))$. However, this rules out, e.g., cases where $Y$ causes $X$ \citep{scholkopf2012causal, lipton2018detecting}. As we have seen, CISA matches the previous notion in the confounded descendant case, and also allows a more general notion of invariance.

\parhead{Domain generalization methods} There are many methods for domain generalization; we give a categorization in the introduction. There have been a number of empirically-oriented surveys testing domain generalization methods in natural settings \citep{koh2021wilds, wiles2021fine, gulrajani2020search}. These find that no method consistently beats ERM, but many methods work well in at least some situations. Our aim here is to give theoretical insight into when each might work.

\parhead{Robust methods}\label{sec:related-work-robust} The CISA framework is reasonable for many real-world problems, but certainly not all. There are other notions of domain shifts and differently motivated methods that do not fit under this framework. For example, many works \citep[e.g.,][]{sagawa2019distributionally, liu2021just, ben2022pada, eastwood2022probable} assume the testing domains are not too different from the training domains (e.g., test samples are drawn from the mixture of training distributions).%
 These methods are complementary to the invariant representation learning approaches we study.%

\section{Conclusion}
In this paper, we have studied whether invariant representation learning methods actually learn invariant structure, and whether this structure is useful for robust prediction.
To do so, we have introduced a notion of domain shift that admits a cannonical notion of ``invariant structure''.
We have seen that data augmentation, distributional invariance learning, and risk invariant learning can be (sometimes) understood as relaxations of this notion.
Whether this holds---and whether the resulting invariant structure is useful---depends on the underlying causal structure of the problem.

As an additional demonstration of the importance of getting the causal structure right, in \cref{sec:experiment} we conduct a simple experiment showing that invariant risk minimization fails when the causal structure is anti-causal. Remarkably, the canonical IRM demonstration example---colored MNIST---has anti-causal structure. We show that simply changing $\Pr(Y)$ across environments causes IRM to fail in (an even simpler version of) this example.

The results in this paper suggest some practical guidance for building robust predictors---e.g., if $Y$ causes $X$, then use methods that enforce $\phi(X) \independent E | Y$. However, this guidance is predicated on CISA domain shifts. Although CISA is quite general as a notion of ``domain shift that admits an invariant structure'', it remains an open question how often real-world domain shifts actually admit invariant structure.

\newpage
\printbibliography

\newpage
\appendix

\section{Experimental Demonstration}\label{sec:experiment}
A key point in the above analysis is that the invariant representation learning method we use must match the underlying  causal structure of the data.
In \Cref{sec:insights}, we discuss this point in the context of existing large-scale, real-world experiments.
However, such demonstrations rely on out of domain performance of various methods, which depends on both the match to the true causal structure but also on
lower level implementation issues---e.g., hyperparameter tuning, overfitting, or optimization.
This can make it somewhat difficult to draw precise conclusions.

To demonstrate the role of causal structure clearly we now study it in a simple case: two bit environments from \cite{kamath2021does}.
\footnote{Code would be made available upon accepted as a conference paper. }
This is a toy setting mimicing the well-known colored MNIST example used to demonstrate Invariant Risk Minimization \cite{arjovsky2019invariant}.
We create two domain shifts. The first is anti-causal and the second is confounded-descendant.
As predicted from \Cref{thm:distributional-signature}, we find that IRM fails but gIRM works in the first case, and vice versa in the second. \footnote{Our experiment uses IRMv1 (and analogously gIRMv1), which is shown to fail with some choices of $\alpha$ \cite{kamath2021does} because of the relaxation from IRM to IRMv1. We avoid those choices of $\alpha$ so that we can focus on the high-level question instead of being distracted by the fragility of IRMv1. Below we use IRM and IRMv1 (also gIRM and gIRMv1) exchangebly.}

In short, we find that even in the simplest possible case, causal structure plays a key role.
In particular, vanilla IRM fails totally on an apparently innocuous modification of the data generating process (prior shift), and this is readily fixed by modifying the method to match the correct causal structure.
Further, the experiments demonstrate that importance sampling based on $Y$---an apparently innocent technique---actually has causal implications and can destroy invariant relaltionship under certain causal structures (e.g. confounded-descendant in the second example).

\subsection{two-bit-envs (anti-causal)}\label{sec:two-bit-anti-causal}

For each domain $e \in \E$, the data generating process is as follows:
\begin{align*}
  & Y \leftarrow \text{Rad}(\gamma_e)\\
  & X_1 \leftarrow Y \cdot \text{Rad}(\alpha)\\
  & X_2 \leftarrow Y \cdot \text{Rad}(\beta_e)
\end{align*}
where $\text{Rad}(\pi)$ is a random variable taking value $-1$ with probability $\pi$ and $+1$ with probability $1 - \pi$. 

This is a simplification of the ColoredMNIST problem (\cite{arjovsky2019invariant}): we know $X_1$ (corresponds to the digit shape) and $Y$ have invariant relationship: $P(X_1 = Y) = 1 - \alpha$. The correlation between $X_2$ (corresponds to color) and $Y$ is spurious, as $P(X_2 = Y) = 1 - \beta_e$ that varies across domains. The label imbalance across domains is due to prior-shift: $P(Y = -1) = \gamma_e$.  We observe 4 training domains and predict on 1 test domain. We use $\alpha = 0.25$, and set $\beta_e = 0.1, 0.2, 0.15, 0.05$ in the training domains respectively, so that using the spurious correlation could get better in-domain performance. However, in the test domain the spurious correlation is flipped: we use $\beta_e = 0.9$ so the out-of-domain performance would be very bad if $X_2$ is used. Finally, we use $\gamma_e = 0.9, 0.1, 0.7, 0.3$ in training domains to create prior-shift. In the test domain the label is balanced ($\gamma_e = 0.5$). These 4 domains constitute $\E$. The goal is to find a predictor $f \in \{ \{+1, - 1\}^2 \rightarrow \mathcal{R} \}$ (prediction is $\hat{y} := \text{sign}(f(x))$ for data $(x, y)$). The optimal 75\% test accuracy is obtained when $f$ satisfies $f(1, \cdot) > 0$ and $f(-1, \cdot) \leq 0$.

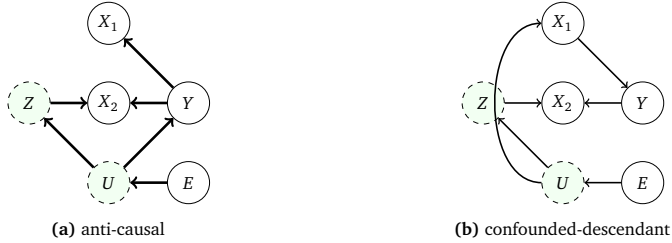
\begin{figure}%
  \captionsetup{width=\columnwidth}
  \begin{subfigure}[b]{.45\columnwidth}
      \centering
      \scalebox{0.7}{
        \begin{tikzpicture}[
    var/.style={draw,circle,inner sep=0pt,minimum size=0.8cm},
    latentconf/.style={draw,circle,dashed,inner sep=0pt,minimum size=0.8cm, fill=green!5}
    ]
        \node (X) [var] {$X_2$};
        \node (Y) [var, right=0.7cm of X] {$Y$};
        \node (Z) [latentconf, left=0.7cm of X] {$Z$};
        \node (Xzperp) [var, above=0.7cm of X] {$X_1$};
        \node (U) [latentconf, below=0.7cm of X] {$U$};
        \node (E) [var, below=0.7cm of Y] {$E$};
        \path[->, line width=0.5mm] (Z) edge (X); 
        \path[->, line width=0.5mm] (U) edge (Z); 
        \path[->, line width=0.5mm] (U) edge (Y); 
        \path[->, line width=0.5mm] (Y) edge (X); 
        \path[->, line width=0.5mm] (Y) edge (Xzperp); 
        \path[->, line width=0.5mm] (E) edge (U); 

    \end{tikzpicture}}
      \caption{anti-causal}
      \label{fig:two-bits-a}
  \end{subfigure}
  \begin{subfigure}[b]{.45\columnwidth}
    \centering
    \scalebox{0.7}{
      \begin{tikzpicture}[
    var/.style={draw,circle,inner sep=0pt,minimum size=0.8cm},
    latentconf/.style={draw,circle,dashed,inner sep=0pt,minimum size=0.8cm, fill=green!5}
    ]
        \node (X) [var] {$X_2$};
        \node (Y) [var, right=0.7cm of X] {$Y$};
        \node (Z) [latentconf, left=0.7cm of X] {$Z$};
        \node (Xzperp) [var, above=0.7cm of X] {$X_1$};
        \node (U) [latentconf, below=0.7cm of X] {$U$};
        \node (E) [var, below=0.7cm of Y] {$E$};
        \path[->, line width=0.3mm] (Z) edge (X); 
        \path[->, line width=0.3mm] (E) edge (U); 
        \path[->, line width=0.3mm] (Xzperp) edge (Y);
        \path[->, bend left=90, line width=0.3mm] (U) edge (Xzperp); 
        \path[->, line width=0.3mm] (U) edge (Z); 
        \path[->, line width=0.3mm] (Y) edge (X); 

    \end{tikzpicture} }
    \caption{confounded-descendant}
    \label{fig:two-bits-c}
\end{subfigure}
  \caption{Causal graphs that are both are CISA-compatible, and can explain the data generating process (for \Cref{sec:two-bit-anti-causal} and \Cref{sec:two-bit-conf-desc} respectively). From the original data generating process, we introduced extra variables $U, Z$ and set $Z \leftarrow U$ and $U \leftarrow E$.}  
    \vspace{-15pt}
\end{figure}

The original two-bit-envs problem has $\gamma_e = 0.5$ across domains, and IRM obtains one optimal predictor. However, it fails in this modified problem with prior-shift (\Cref{tab:anti-causal}). To understand and fix its failure, we can find a causal interpretation that both explains the data and is CISA-compatible: introduce $U, Z$ and set $Z \leftarrow U; U \leftarrow E$ as shown in \Cref{fig:two-bits-a}. Then this domain shift problem falls under anti-causal subtype, and a CF-invariant predictor should only rely on $X_1$. IRM does not enforce the right invariance and fails to remove spurious $X_2$ as a result. Instead, we should use gIRM to (partially) enforce CF-invariance. Indeed, gIRM successfully forces the model to discard $X_2$ and obtain the optimal test accuracy (\Cref{tab:anti-causal}). 
Note that in the original two-bit-envs problem without prior-shift, both $X_1 \leftarrow Y \cdot \text{Rad}(\alpha)$ and $Y \leftarrow X_1 \cdot \text{Rad}(\alpha)$ can explain the data. Therefore we we can \emph{interpret} the data generating process as either anti-causal or confounded-descendant. So both IRM and gIRM (partially) enforce CF-invariance --- in fact $\Phi_{\text{IRM}}(\E) = \Phi_{\text{g-IRM}}(\E)$ so the resulting predictor is the same.

\begin{table}
\captionsetup{width=\columnwidth}
\caption{In two-bit (anti-causal) experiment, IRMv1 predictor fails to discard spurious features because of prior-shift in $Y$. However, after modifying the method to match the underlying causal structure (using gIRMv1), we can recovers a CF-invariant predictor that obtains optimal test accuracy. The results are under cross-entropy loss (similar for for squared loss). }
$$
\begin{array}{|c|c|c|c|c|}
\hline & \multicolumn{2}{|c|}{f_{\mathrm{IRMv1}}} & \multicolumn{2}{c|}{f_{\mathrm{gIRMv1}}} \\
\cline { 2 - 5 } & X_{2}=+1 & X_{2}=-1 & X_{2}=+1 & X_{2}=-1 \\
\hline X_{1}=+1 & 2.53 & -0.93 & 1.16 & 1.08 \\
\hline X_{1}=-1 & -0.08 & -3.19 & -1.11 & -1.03 \\
\hline
\end{array}
$$
\label{tab:anti-causal}
\end{table}

\begin{table}
  \captionsetup{width=\columnwidth}
  \caption{In two-bit (confounded-descendant) experiment, directly applying IRMv1 recovers a CF-invariant predictor that obtains optimal test accuracy. On the contrary, applying importance sampling can destroy the invariant relationship in data --- as a result, gIRMv1 learns only the trivial invariant predictor. 
  The results are under cross-entropy loss (similar for for squared loss).}
  $$
  \begin{array}{|c|c|c|c|c|}
  \hline & \multicolumn{2}{|c|}{f_{\mathrm{IRMv1}}} & \multicolumn{2}{c|}{f_{\mathrm{gIRMv1}}} \\
  \cline { 2 - 5 } & X_{2}=+1 & X_{2}=-1 & X_{2}=+1 & X_{2}=-1 \\
  \hline X_{1}=+1 & 1.1 & 1.1 & 0 & 0 \\
  \hline X_{1}=-1 & -1.1 & -1.1 & 0 & 0 \\
  \hline
  \end{array}
  $$
  \label{tab:conf-descend}
  \end{table}

\subsection{two-bit-envs (conf-desc)}\label{sec:two-bit-conf-desc}

Our implementation of enforcing gIRM regularization \footnote{There are two ways to implement gIRM. The first way is apply importance sampling to the regularization term only; the second way is to apply it to both the loss term and regularization term. In these two examples, the two implementations give the same result. Thus to better illustrate our point on importance sampling, we study the second implementation.} is equivalent to: first, perform importance sampling with weight $w_e((x_1, x_2), y) = \frac{P_0(y)}{P_e(y)}$; next, enforce IRM regularization. Similarly, it's a common practice to perform importance sampling based on label $Y$ when it's imbalanced. However, as we shall show in this example, importance sampling can remove invariant features under certain causal structures. 

In this example, the data generating process is as follows: for each domain $e \in \E$, 
\begin{align*}
  & X_1 \leftarrow \text{Rad}(\gamma_e)\\
  & Y \leftarrow X_1 \cdot \text{Rad}(\alpha)\\
  & X_2 \leftarrow Y \cdot \text{Rad}(\beta_e)
\end{align*}
Compared to the previous example, we do not change $P(X_1 = Y)$ and $P(X_2 = Y)$. The only change is on how the label imbalance is created: through covariate shift in $X_1$. We use the same parameters for $\alpha, \beta_e, \gamma_e$. 

Again $P_e(Y)$ is different across domains, but this time gIRM forces the model to always predict $0$ (\Cref{tab:conf-descend})! To understand why, we find a CISA-compatible DAG that explains the data as shown in \Cref{fig:two-bits-c} (similarly introduce $U, Z$ and set $Z \leftarrow U; U \leftarrow E$). This is confounded-descendant and IRM enforces the right invariance whereas gIRM enforces the wrong one.

To understand how the importance sampling destroys even the invariant relationship between $X_1$ and $Y$, we look at the target distribution after reweighting (call it $Q_e$). Since the weighting function is $w_e(x_1, y) = \frac{P_0(x_1)}{P_e(y)}$ and that $w_e(x_1, y) = \frac{Q_e(x_1, y)}{P_e(x_1, y)}$, we have $Q_e(x_1, y) = P_e(x_1, y) \frac{P_0(y)}{P_e(y)}$. Observe that the probability $Q_e(X_1 = Y) = g(\gamma_e)$ where the $[0, 1]$-supported function $g$ (treat $\alpha$ as a constant and assume $\alpha > 0.5$) satisfies the following:
\begin{compactenum}
  \item $g(\gamma) = g(1 - \gamma)$ so $g$ is symmetric around $0.5$. 
  \item $g$ strictly increases on $[0, 0.5]$ and strictly decreases on $[0.5, 1]$
  \item $g(0) = g(1) = 0.5$, and $g(0.5) = 1 - \alpha$, so $g$ decreases as $\gamma_e$ deviates from $0.5$
\end{compactenum}
Thus $0.5 < Q_1(X_1 = Y) = Q_2(X_1 = Y) < Q_3(X_1 = Y) = Q_4(X_1 = Y) < 1 - \alpha$. Therefore, importance sampling not only weakens the relalationship between $X_1$ and $Y$, but also makese it unstable! As a result, enforcing IRM on the resampled distribution finds no non-trivial invariant predictors.

\section{Proofs}

\dataaug*
\begin{proof}

  First, for the convenience of notation let's assume $X = X(z_0) \ a.e.$ for some $z_0 \in \mathcal{Z}$. Then by the label-preserving $\mathcal{T}$, we have: for each $t \in \mathcal{T}$ we have $t(X) (= t(X(z_0))) = X(z)$ for some $z \in \mathcal{Z}$. 

  Consider consistency training. Let $\Phi_{\text{c}}(\mathcal{T})$ denote the set of representation functions satisfying consistency requirement under transformation set $\mathcal{T}$, i.e. $\Phi_{\text{c}}(\mathcal{T}) := \{\phi: \phi(X) = \phi(t(X)) \ a.e. \ \forall t \in \mathcal{T} \}$. If $\phi \in \Phi_{\text{c}}(\mathcal{T})$, then for any $z, z' \in \mathcal{Z}$, can find $t \in \mathcal{T}$ such that $X(z') = t(X(z))$ since $\mathcal{T}$ enumerates all potential outcomes of $Z$; therefore $\phi(X(z')) = \phi(t(X(z))) = \phi(X(z)) \ a.e.$ by consistency requirement. Thus $\phi \in \Phi_{\text{cf-inv}}(\E)$. On the other hand if $\phi \in \Phi_{\text{cf-inv}}(\E)$, then for any $t \in \mathcal{T}$, we have $\phi(t(X)) = \phi(X(z)) = \phi(X)$ for some $z \in \mathcal{Z}$. Thus $\phi \in \Phi_{\text{c}}(\mathcal{T})$. Therefore $\Phi_{\text{c}}(\mathcal{T}) = \Phi_{\text{cf-inv}}(\E)$. Therefore, training the model to minimize risk on original data, with hard consistency regularization is equivalent to CF-invariant representation learning, which recovers the optimal CF-invariant predictor on training distribution. 
  
  Consider ERM training on augmented data with purely-spurious $Z$. Let $P$ denote the original distribution, and $\Tilde{P}$ denote the distribution after the augmentation. Let $T$ be the random variable for transformation operation. First, the generating process of the augmented data is: first sample $T \sim \Tilde{P}_T(.)$; then sample $(X, Y)|T = t$ from the distribution of $(t(X), Y)$. Then we have:
  \begin{align*}
    \Tilde{P}(X, Y) & = \int P(t(X), Y) d \Tilde{P}_T(t) \\
    & = \int P(X(z), Y) d \Tilde{P}_Z(z)\\
    & = \int P(X(z), Y(z)) d \Tilde{P}_Z(z)\\
    & = \int P(X, Y | do(z)) d \Tilde{P}_Z(z)
  \end{align*}
  by the label-preserving of $\mathcal{T}$, and the fact that $Y$ is not a descendant of $Z$.  

  Next, observe that $P(y | x, do(z))= P(y | x_z^{\perp})$. This is because: in original probability we have $Y \independent X | X_z^{\perp}, Z$; the $do(z)$-operation removes the incoming edges of $Z$ and set $Z = z$; as a result $P(y | x, do(z)) = P(y | x_z^{\perp}, do(z)) = P(y | x_z^{\perp})$. 

  Put together:
  \begin{align*}
        \Tilde{P}(X, Y) & = \int P(X, Y | do(z)) d \Tilde{P}(z)\\
        & = \int P(Y | X, do(z)) P(X | do(z)) d \Tilde{P}(z)\\
        & = \int P(Y | X_z^{\perp}) P(X | do(z)) d \Tilde{P}(z)\\
        & = P(Y | X_z^{\perp}) \int P(X | do(z)) d \Tilde{P}(z) = P(Y | X_z^{\perp})\Tilde{P}(X)
  \end{align*}

  Therefore the objective is: 
  $$E_{\Tilde{P}}[L(Y, f(X))] = E_{\Tilde{P}(X)}[E_{P(Y | X_z^{\perp})}(L(Y, f(X)))]$$ 
  Then for any input $x$, the the optimal predictor output $f^{*}(x) = \text{argmin}_{a(x)} \int L(y, a(x)) dP(y | x_z^{\perp})$. This is the same as directly restricting predictor to be CF-invariant. 
\end{proof}

\classifications*

\begin{proof}
  \begin{figure}%
    \captionsetup{width=\columnwidth}
    \begin{subfigure}[b]{.5\textwidth}
      \centering
      \scalebox{0.7}{
        \begin{tikzpicture}[
var/.style={draw,circle,inner sep=0pt,minimum size=0.8cm},
latentconf/.style={draw,circle,dashed,inner sep=0pt,minimum size=0.8cm, fill=green!5}
]
    \node (X) [var] {$X_z$};
    \node (Y) [var, right=0.7cm of X] {$Y$};
    \node (Z) [latentconf, left=0.7cm of X] {$Z$};
    \node (Xzperp) [var, above=0.7cm of X] {$X_z^{\perp}$};
    \node (U) [latentconf, below=0.7cm of X] {$U$};
    \node (E) [var, below=0.7cm of Y] {$E$};
    \path[->, line width=0.3mm] (Z) edge (X); 
    \path[->, line width=0.3mm] (Xzperp) edge (X);
    \path[->, line width=0.3mm] (E) edge (U); 
    \path[->, orange, line width=0.3mm] (Y) edge (X); 

    \tikzstyle{bigbox}=[draw, dotted, very thick, inner sep=8pt]
    \node[bigbox, fit=(X)(Xzperp)] (Xfull) {};
    \node[below right] at (Xfull.north west) {$X$};
\end{tikzpicture} }
      \caption{base graph}
      \label{fig2:dags-types-base}
  \end{subfigure}
      \begin{subfigure}[b]{.5\textwidth}
        \centering
        \scalebox{0.7}{
          \begin{tikzpicture}[
var/.style={draw,circle,inner sep=0pt,minimum size=0.8cm},
latentconf/.style={draw,circle,dashed,inner sep=0pt,minimum size=0.8cm, fill=green!5}
]
    \node (X) [var] {$X_z$};
    \node (Y) [var, right=0.7cm of X] {$Y$};
    \node (Z) [latentconf, left=0.7cm of X] {$Z$};
    \node (Xzperp) [var, above=0.7cm of X] {$X_z^{\perp}$};
    \node (U) [latentconf, below=0.7cm of X] {$U$};
    \node (E) [var, below=0.7cm of Y] {$E$};
    \path[->, line width=0.3mm] (Z) edge (X); 
    \path[->, line width=0.3mm] (Xzperp) edge (X);
    \path[->, line width=0.3mm] (E) edge (U); 
    \path[->, orange, line width=0.3mm] (U) edge (X); 
    \path[->, orange, line width=0.3mm] (Y) edge (X); 
    \path[->, orange, bend left=45, line width=0.3mm] (Y) edge (Z); 
    \path[->, blue, line width=0.3mm, draw opacity=0.5] (Y) edge (Xzperp);
    \path[->, blue, line width=0.3mm, draw opacity=0.5] (U) edge (Y); 
    \path[->, blue, line width=0.3mm, draw opacity=0.5] (U) edge (Z); 
    \path[->, orange, line width=0.3mm] (Xzperp) edge (Z); 

    \tikzstyle{bigbox}=[draw, dotted, very thick, inner sep=8pt]
    \node[bigbox, fit=(X)(Xzperp)] (Xfull) {};
    \node[below right] at (Xfull.north west) {$X$};
\end{tikzpicture} }
        \caption{anti-causal}
        \label{fig2:dags-types-a}
    \end{subfigure}
    \begin{subfigure}[b]{.5\textwidth}
        \centering
        \scalebox{0.7}{
        \begin{tikzpicture}[
var/.style={draw,circle,inner sep=0pt,minimum size=0.8cm},
latentconf/.style={draw,circle,dashed,inner sep=0pt,minimum size=0.8cm, fill=green!5}
]
    \node (X) [var] {$X_z$};
    \node (Y) [var, right=0.7cm of X] {$Y$};
    \node (Z) [latentconf, left=0.7cm of X] {$Z$};
    \node (Xzperp) [var, above=0.7cm of X] {$X_z^{\perp}$};
    \node (U) [latentconf, below=0.7cm of X] {$U$};
    \node (E) [var, below=0.7cm of Y] {$E$};
    \path[->, line width=0.3mm] (Z) edge (X); 
    \path[->, line width=0.3mm] (Xzperp) edge (X);
    \path[->, line width=0.3mm] (E) edge (U); 
    \path[->, orange, line width=0.3mm] (U) edge (X); 
    \path[->, orange, line width=0.3mm] (Y) edge (X); 
    \path[->, orange, bend left=45, line width=0.3mm] (Y) edge (Z); 
    \path[->, blue, line width=0.3mm, draw opacity=0.5] (Xzperp) edge (Y);
    \path[->, blue, line width=0.3mm, draw opacity=0.5] (U) edge (Y); 
    \path[->, blue, line width=0.3mm, draw opacity=0.5] (U) edge (Z); 
    \path[->, orange, line width=0.3mm] (Xzperp) edge (Z); 

    \tikzstyle{bigbox}=[draw, dotted, very thick, inner sep=8pt]
    \node[bigbox, fit=(X)(Xzperp)] (Xfull) {};
    \node[below right] at (Xfull.north west) {$X$};
\end{tikzpicture} }
        \caption{confounded-outcome}
        \label{fig2:dags-types-b}
    \end{subfigure}
    \begin{subfigure}[b]{.5\textwidth}
        \centering
        \scalebox{0.7}{    
        \begin{tikzpicture}[
var/.style={draw,circle,inner sep=0pt,minimum size=0.8cm},
latentconf/.style={draw,circle,dashed,inner sep=0pt,minimum size=0.8cm, fill=green!5}
]
    \node (X) [var] {$X_z$};
    \node (Y) [var, right=0.7cm of X] {$Y$};
    \node (Z) [latentconf, left=0.7cm of X] {$Z$};
    \node (Xzperp) [var, above=0.7cm of X] {$X_z^{\perp}$};
    \node (U) [latentconf, below=0.7cm of X] {$U$};
    \node (E) [var, below=0.7cm of Y] {$E$};
    \path[->, line width=0.3mm] (Z) edge (X); 
    \path[->, line width=0.3mm] (Xzperp) edge (X);
    \path[->, line width=0.3mm] (E) edge (U); 
    \path[->, orange, line width=0.3mm] (U) edge (X); 
    \path[->, orange, line width=0.3mm] (Y) edge (X); 
    \path[->, orange, bend left=45, line width=0.3mm] (Y) edge (Z); 
    \path[->, blue, line width=0.3mm, draw opacity=0.5] (Xzperp) edge (Y);
    \path[->, blue, bend left=90, line width=0.3mm, draw opacity=0.5] (U) edge (Xzperp); 
    \path[->, blue, dashed, line width=0.3mm, draw opacity=0.5] (U) edge (Z); 
    \path[->, blue, dashed, line width=0.3mm, draw opacity=0.5] (Xzperp) edge (Z); 

    \tikzstyle{bigbox}=[draw, dotted, very thick, inner sep=8pt]
    \node[bigbox, fit=(X)(Xzperp)] (Xfull) {};
    \node[below right] at (Xfull.north west) {$X$};
\end{tikzpicture} }
        \caption{confounded-descendant}
        \label{fig2:dags-types-c}
    \end{subfigure}
    \caption{We put below \Cref{fig:dags-types} again and the base graph for convenience of inspection.The black arrows are included in all graphs. The blue arrows are specific to different causal structures. The orange arrows are optional. At least one of the two dashed blue arrows in \Cref{fig2:dags-types-c} must exist.}  
    \label{fig2:dags-types}
         \vspace{-15pt}
  \end{figure}
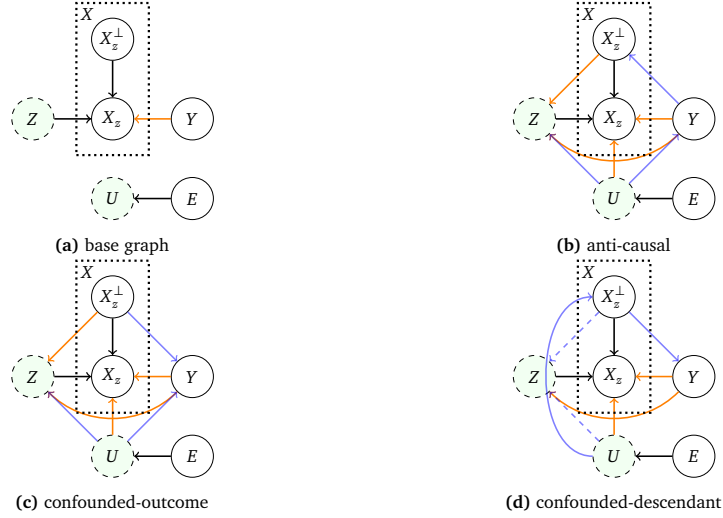

  There are a finite number of possible causal DAGs relating the variables $U,Z,X_z,Y,X_z^{\perp}$. Moreover, for a DAG to be compatible with CISA it must satisfy some conditions that narrows down the set. 
  In particular, $Z$ causes $X_z$ but not $X_z^{\perp}$ or $Y$ (by definition of $Z$); $U$ should confound $Z$ and $Y$, but cannot confound $X_z^{\perp}$ and $Y$ (by definition of $U$); $E$ only affects $U$ (as only $P_e(U)$ changes across environments and $P(X,Y, Z | U)$ is invariant). Below we use these conditions to enumerate \ref{fig2:categorization} all CISA-compatible DAGs. 

  Now we have $Z \rightarrow X_z$, $X_z^{\perp} \rightarrow X$ and $E \rightarrow U$ (and it's the only edge from $E$). Note that we define $Z$ to not have any causal effect on $Y$. Accordingly, the path $Z \rightarrow X_z \rightarrow Y$ is ruled out. Thus $X_z \rightarrow Y$ is not allowed but $Y \rightarrow X_z$ is optional. These edges form the base graph \ref{fig2:dags-types-base} to build upon.

  Next, we divide into two cases: $X_z^{\perp} \rightarrow Y$ and $X_z^{\perp} \leftarrow Y$ (when there is no edge between them, we can treat it as either case and the resulting graphs are the same). 

  When $X_z^{\perp} \leftarrow Y$: we require that $U$ confounds $Z, Y$, so we need $U \rightarrow Y$ (otherwise $U$ can't cause $Y$) and $U \rightarrow Z$ (otherwise $U$ can't confound the relationship between $U, Z$). We do not allow $U \rightarrow X_z^{\perp}$ as otherwise the relationship between $X_z^{\perp}$ and $Y$ is confounded. There are a few optional edges $U \rightarrow X_z, Y \rightarrow Z, X_z^{\perp} \rightarrow Z$, as they do not violate CISA assumptions. Other edges cannot be allowed as they will violate CISA assumptions. These constitute the anti-causal subtype as illustrated in \ref{fig2:dags-types-a}. 

  When $X_z^{\perp} \rightarrow Y$, we can again divide into two exclusive cases: $U \rightarrow Y$ and $U \rightarrow X_z^{\perp}$. Why? We need at least one of these two edges, as otherwise $U$ does not cause $Y$; the two edges cannot exist simultaneously as otherwise the relationship between $X_z^{\perp}$ and $Y$ is confounded. 

  So, when $X_z^{\perp} \rightarrow Y$ and $U \rightarrow Y$: we need $U \rightarrow Z$ as otherwise $U$ does not confound $Z, Y$. There are a few optional edges $U \rightarrow X_z, Y \rightarrow Z, X_z^{\perp} \rightarrow Z$, as they do not violate CISA assumptions. Other edges cannot be allowed as they will violate CISA assumptions. These constitute the confounded-outcome subtype as illustrated in \ref{fig2:dags-types-b}. 

  Next, when $X_z^{\perp} \rightarrow Y$ and $U \rightarrow X_z^{\perp}$: to let $U$ cause $Z$ and confound $Z, Y$, we need at least one of the two edges (or both) $X_z^{\perp} \rightarrow Z$, $U \rightarrow Z$.There are a few optional edges $U \rightarrow X_z, Y \rightarrow Z$, as they do not violate CISA assumptions. Other edges cannot be allowed as they will violate CISA assumptions. These constitute the confounded-descendant subtype as illustrated in \ref{fig2:dags-types-c}.

  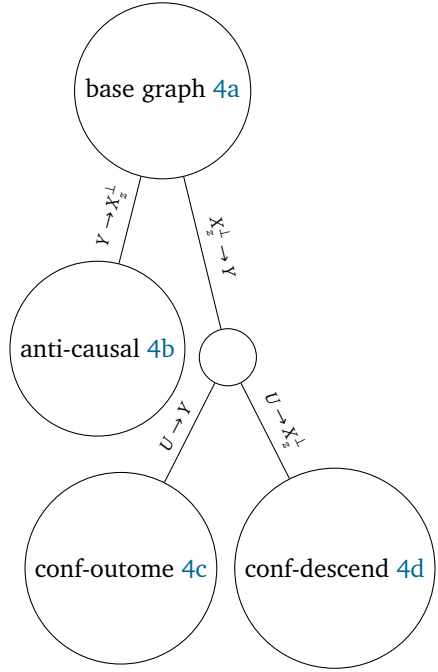
\begin{figure}
    \centering
    \caption{Enumerating CISA-compatible causal graphs}
    \medskip
    \begin{forest}
      for tree = {circle, 
          draw,
          minimum width = 2.25em,
          l sep+=2em
      }
        [base graph \ref{fig2:dags-types-base},
          [anti-causal \ref{fig2:dags-types-a}, edge label={node[midway,above, sloped, font=\scriptsize]{$Y \rightarrow X_z^{\perp}$}} ]
          [, edge label={node[midway,above,sloped, font=\scriptsize]{$X_z^{\perp} \rightarrow Y$}} 
            [conf-outome \ref{fig2:dags-types-b}, edge label={node[midway,above,sloped, font=\scriptsize]{$U \rightarrow Y$}}]
            [conf-descend \ref{fig2:dags-types-c}, edge label={node[midway,above,sloped, font=\scriptsize]{$U \rightarrow X_z^{\perp}$}}] ]
        ]
    \end{forest}
    \label{fig2:categorization}
    \end{figure}

\end{proof}

\signature*

\begin{proof}
  Reading d-separation from the corresponding DAGs, we have $X_z^{\perp} \perp E | Y$ for anti-causal problems; $X_z^{\perp} \perp E$ for confounded-outcome problems; $Y \perp E | X_z^{\perp}$ for confounded-descendant problems. Since $\phi$ is CF-invariant, that means $\phi(X)$ is $X_z^{\perp}$-measurable. Thus the claim follows. 
\end{proof}

\irm*

\begin{proof}
Confounded-descendant case: let $\phi \in \Phi_{\text{DI}}(\E)$, i.e. $Y \independent E | \phi(X)$. To show the risk minimizer is the same, it suffices to show $P_e(Y | \phi(X))$ to be the same for all $e \in \E$. This is immediate from the distributional invariance. 

Anti-causal case: if the representation $\phi \in \Phi_{\text{DI}}(\E)$, i.e. $\phi(X) \independent E | Y$, 
\begin{align*}
    & E_{P_e}[\frac{P_0(Y)}{P_e(Y)} L(Y, (\bar{w} \circ \phi)(X))] \\
    &= E_{Y \sim P_e} [\frac{P_0(Y)}{P_e(Y)} [E_{\phi(X) \sim P_e(.|Y)}(L(Y, (\bar{w} \circ \phi)(X)) | Y)]]\\
    & = E_{Y \sim P_0} [E_{\phi(X) \sim P(.|Y)}(L(Y, (\bar{w} \circ \phi)(X)) | Y)]
\end{align*}
The second equality is because $\phi(X) \independent E | Y$.

Thus the objective function is the same across domains, so the optimal $w$ is the same. Therefore $\phi \in \Phi_{\text{g-IRM}}(\E)$
\end{proof}

\end{document}